\documentclass[12pt,oneside,reqno]{amsart}
\usepackage{graphicx}

\graphicspath{ {figs/} }

\usepackage{pict2e}
\usepackage{amssymb}
\usepackage{amsthm}                
\usepackage[margin=0.9in]{geometry}  

\usepackage{epstopdf}
\usepackage{amsmath}
\usepackage{graphicx}
\usepackage{subfigure}
\usepackage{latexsym}
\usepackage{amsmath}
\usepackage{amsfonts}
\usepackage{amssymb}
\usepackage{tikz}
\usepackage{mathdots}
\usepackage{comment}
\usepackage{float}
\usepackage[mathscr]{euscript}
\usepackage{enumerate}
\usepackage{epsfig}
\usepackage { hyperref }
\usepackage { graphics }
\usepackage { graphicx }
\usepackage{MnSymbol}

\usetikzlibrary{arrows.meta}
\usetikzlibrary{knots}
\usetikzlibrary{hobby}
\usetikzlibrary{arrows,decorations.markings}

\usepackage[utf8]{inputenc}
\usepackage{longtable}
\usepackage[lite]{amsrefs}

\renewcommand{\PrintDOI}[1]{\href{http://dx.doi.org/\detokenize{#1}}{doi: \detokenize{#1}}%
	\IfEmptyBibField{pages}{, (to appear in print)}{}}

\theoremstyle{definition}
\newtheorem{theorem}{Theorem}[section]

\newtheorem{proposition}[theorem]{Proposition}

\theoremstyle{definition}
\newtheorem{definition}[theorem]{Definition}

\theoremstyle{remark}

\numberwithin{equation}{section}

\numberwithin{equation}{section}

\title{Topological Deep Learning: Classification Neural Networks}

\author{Mustafa Hajij}
\address{Santa Clara University,Santa Clara, USA}
\email{hajij@scu.edu}
\author{Kyle Istvan}
\email{kyleistvan@gmail.edu}

\begin{document}

\maketitle

\begin{abstract}
Topological deep learning is a formalism that is aimed at introducing topological language to deep learning for the purpose of utilizing the minimal mathematical structures to formalize problems that arise in a generic deep learning problem. This is the first of a sequence of articles with the purpose of introducing and studying this formalism. In this article, we define and study the classification problem in machine learning in a topological setting. Using this topological framework, we show when the classification problem is possible or not possible in the context of neural networks.  Finally, we show that for a given data, the architecture of a classification neural network must take into account the topology of this data in order to achieve a successful classification task.
\end{abstract}

\section{Introduction}
Recent years have witnessed increased interest in the role topology plays in machine learning and data science \cite{carlsson2009topology}. Topology is a natural tool that allows the formulation of many longstanding problems in these fields. For instance, \textit{persistent homology} \cite{edelsbrunner2010computational} has been overwhelmingly successful at finding solutions to a vast array of complex data problems \cite{attene2003shape, bajaj1997contour, boyell1963hybrid, carr2004simplifying, curto2017can, DabaghianMemoliFrank2012, giusti2016two, kweon1994extracting, LeeChungKang2011b, LeeChungKang2011, LeeKangChung2012, LeeKangChung2012b, lum2013extracting, nicolau2011topology, rosen2017using}.

On the other hand, the role that topology plays in deep learning is still mostly restricted to techniques that attempt to enhance machine learning models \cite{hofer2017deep,bruel2019topology,wangtopogan}. However, we believe that topology can and will play a central role in deep learning and AI in general. This is the first of a sequence of articles with the purpose of introducing \textit{topological deep learning}, a formalism that is aimed at introducing topological language to deep learning for the purpose of utilizing the minimal mathematical structures to  formalize problems that arise in a generic deep learning problem. 

In this article we define and study the classification problem in a topological setting. Using this topological machinery, we show when the classification problem is possible or not possible in the context of neural networks. Finally, we show how the architecture of a neural network cannot be chosen independently from the topology of the underlying data. To demonstrate these results, we provide an example dataset and show how it is acted upon by a neural net from this topological perspective. A more thorough treatment of the topic presented here is given in \cite{hajij2020topology}.

\section{Previous Work}

 The earliest hints, that we know of, related to our work appears in a blog by C. Olah \cite{olah2014neural}. Olah performed a number of topological experiments illustrating the importance of considering the topology of the underlying data when making a neural network. In \cite{naitzat2020topology} the activations of a binary classification neural network were considered as point clouds that the layer functions of the network are acting on. The topologies of these activations are then studied using homological tools such as persistent homology \cite{edelsbrunner2010computational}.

Alternatively, our work here can be regarded as part of the effort in the literature regarding the explainablity of deep learning \cite{hagras2018toward,selvaraju2017grad}. The authors Zeiler et. al. in \cite{zeiler2014visualizing} introduced a visualization technique that gives insight into the intermediate layers of convolutional neural networks. In \cite{yosinski2015understanding} also gives a way to visualize and interpret the a given convolutional network by looking at the activations.

\section{Background}
 A \textit{neural network}, or simply a \textit{network}, is a function $Net: \mathbb{R}^ {d_{in}} \longrightarrow \mathbb{R}^{d_{out}}$ defined by a composition of the form: 
\begin{equation}
\label{Net}
    Net:=f_{L} \circ \cdots \circ f_{1}
\end{equation}
where the functions $f_{i}$, $1 \leq i \leq L $ are called the \textit{layer functions}. A layer function $f_i:\mathbb{ R }^{n_i} \longrightarrow \mathbb{ R }^{m_i} $ is typically a continuous, piece-wise smooth function of the following form: $f_i(x)=\sigma(W_i(x)+b_i)$ where $W_i$ is  an $m_i\times n_i$ matrix, $b_i$ is a vector in $\mathbb{R}^{m_i} $, and $\sigma :\mathbb{R}\longrightarrow \mathbb{R} $ is an appropriately chosen nonlinear function that is applied coordinate-wise on an input vector $(z_1,\cdots,z_{m_i} ) $ to get a vector $(\sigma( z_1),\cdots,\sigma(z_{m_i}))$.







 
\section{Data In a Topological Setting}

The purpose of this section is define the notion of data using topological notions.

\subsection{Topological Data}
\label{TD}
Denote by $M^n$ to a manifold $M$ of dimension $n$. Let $D = M_1^{i_1} \cupdot M_2^{i_2} \cdots  \cupdot M_k^{i_k}$ be a disjoint union of $k$ compact manifolds. Let $h: D \to E $ be a continuous function on $D$. We refer to the pair $(D,h)$ as \textit{topological data} and refer to $E$ as the \textit{the ambient space} of the topological data, or simply the ambient space of the data.

A few remarks here must be made about the above definition. First note that the definition above is consistent with the statistical version. The space $E$, usually some Euclidean space, represents the ambient space of a probability distribution $\mu$ from which we sample the data. The support of $\mu$ is $\mathcal{D}:=h(D)$. The assumption that the data lives on a manifold-like structure is justified in the literature  \cite{fefferman2016testing,lei2020geometric}. \footnote{While we make this assumption here, it not strictly necessary anywhere in our proofs.}


\subsection{Topologically Labeled Data}

Let $(D,h)$ be topological data with $h:D\to \mathcal{D}\subset E $.
 Let $\mathcal{Y} = \{l_1, \cdots , l_d\}$ be a finite set. A \textit{topological labeling} on $\mathcal{D}$ is a closed subset $\mathcal{D}_L \subset \mathcal{D} $ 
 along with a
 surjective continuous function $g : \mathcal{D}_L \to \mathcal{Y}$ where $\mathcal{Y}$ is given the discrete topology. The triplet $(D,h,g)$ will be called \textit{topologically labeled data}.

Topologically labeled data is a topological object that corresponds to labeled data in the typical statistical setting for a supervised classification machine learning problem. 

 
 \section{The Topological Classification Problem}
 \label{clfers}
 With the above setting we now demonstrate how to realize the classification problem as a topological problem. In what follows we set $\mathcal{D}_k$ to denote $g^{-1}(l_k)$ for $l_k \in \mathcal{Y}$. 
\begin{definition}
\label{def1}
Let $(D,h,g)$ be topologically labeled data with, $h: D \to \mathcal{D} \subset \mathbb{R}^n$ and $g:\mathcal{D}_L\to \mathcal{Y} $ where $|\mathcal{Y}|=d$. A \textit{topological classifier} on $(D,h,g)$ is a continuous function $f:\mathbb{R}^n \to \mathbb{R}^k$.  We say that $f$ \textit{separates} the topologically labeled data $(D,h,g)$ if we can find $d$ disjoint embedded $k$-dimensional discs $A_1, \cdots , A_{d}$ in $\mathbb{R}^k$ such that $f(\mathcal{D}_d)\subset A_d$.
\end{definition}

In general, a topologically labeled data can be knotted, linked and entangled together in a non-trivial manner by the embedding $h$, and the existence of a function $f$ that separates this data is not immediate. The preceding description is an topological rewording of the classification problem typically given in a statistical setting. Indeed, a successful classifier tries to \textit{separate} the labeled data by mapping the raw input data into another space where this data can be separated easily according to the given class. 


The function $f$ is the learning function that we try to compute, in practice. The first question one could ask in this context is one of existence: given topologically labeled data $(D,h,g)$ when can we find a function $f$ that separates this data? We answer this question next.  

\subsection{Topological Classifiers and Separability of Topologically Labeled Data}

We start with the binary classification problem, namely when $|\mathcal{Y}|=2$. We have the following proposition:

\begin{proposition}
\label{one}
Let $(D,h,g)$ by a topologically labeled data with $h: D \longrightarrow \mathcal{D} \subset  \mathbb{R}^{d_{in}}$ and $g: \mathcal{D}_L  \to \{l_1, l_2\}$. Then there exists a topological classifier $f: \mathbb{R}^{d_{in}} \to \mathbb{R}$ that separates $(D,h,g)$.
\end{proposition}

\begin{proof}
By definition, label function  $g: \mathcal{D}_L \longrightarrow \{l_1, l_2\}$ induces a partition on $\mathcal{D}_L$ into two disjoint closed sets $\mathcal{D}_1:=g^{-1}(l_1)$ and $\mathcal{D}_2:=g^{-1}(l_2)$. By Urysohn's lemma there exists a function $f^*:\mathcal{D} \longrightarrow [0,1]$ such that $f^*(\mathcal{D}_1)=0$ and $f^*(\mathcal{D}_2)=1$. Since $\mathcal{D}$ is closed in $\mathbb{R}^{d_{in}}$ then by Tietze extension theorem there exists an extension of $f^*$ to a continuous function $f : \mathbb{R}^{d_{in}} \to \mathbb{R} $ such that $f^*(\mathcal{D})=f(\mathcal{D})$.  In particular, $f(\mathcal{D}_1)=0$ and $f(\mathcal{D}_2)=1$. Hence the function $f$ separates $(D,h,g)$.
\end{proof}

Proposition \ref{one} can be easily generalized to obtain functions that separate $(D,h,g)$ in any Euclidean space $\mathbb{R}^k$. Namely, for any $k\geq 1$ there exists a continuous map $F: \mathbb{R}^{d_{in}} \longrightarrow \mathbb{R}^{k}$ that separates $(D,h,g)$. This can be done by defining $F = (f_1 , f_2)$ where $f_1:\mathbb{R}^{d_{in}}\longrightarrow [0,1]$ is the continuous function guaranteed by Urysohn's Lemma and $f_2:\mathbb{R}^{d_{in}}\longrightarrow \mathbb{R}^{k-1}$ is an arbitrary continuous function. This function $F$ clearly separates  $(X,h,g)$. We record this fact in the following proposition.
\begin{proposition}
\label{second}
Let $(D,h,g)$ by a topologically labeled data with  $h: D \to \mathcal{D} \subset \mathbb{R}^{d_{in}}$ and $g: \mathcal{D}_L \to \{l_1, l_2\}$. Then for any $k\geq 1$ there exists a continuous map $f: \mathbb{R}^{d_{in}} \to \mathbb{R}^{k}$ that separates $(D,h,g)$.
\end{proposition}

Proposition \ref{second} can be generalized to the case when the set $\mathcal{Y}$ has an arbitrary finite size. This can be done by because Urysohn’s Lemma remains valid when we start with $n$ disjoint sets instead of $2$. The following theorem, which generalizes \ref{second}, asserts the existence of a topological classifier $f$ that separates any given topologically labeled data.

\begin{theorem}
\label{generalization TLD}
Let $(D,h,g)$ be topologically labeled data with  $h: D \to  \mathcal{D} \subset R^{d_{in}}$ and $g: \mathcal{D}_L \to \mathcal{Y}$.  Then there exists a continuous map $f: \mathbb{R}^{d_{in}} \to \mathbb{R}^k$ that separates $(D,h,g)$ for any integer $ k \geq 1$.
\end{theorem}

\section{Neural Networks as Topological Classifiers}

Let $(D,h,g)$ by a topologically labeled data with, $h: D \to \mathcal{D} \subset R^{d_{in}}$ and $g: \mathcal{D}_L  \to \mathcal{Y}=\{l_1, \cdots l_n \}$. Can we find a neural network defined on $R^{d_{in}}$ that separates the data $(D,h,g)$ ? We start by framing the softmax classification networks using  topological terminologies.

Typical, classification neural networks have a special layer function at the end where one uses the \textit{softmax activation function} \footnote{There are other types of classification neural networks but this is beyond the scope of our discussion here}. Denote by $\Delta_n$ the $n^{th}$ simplex as the convex hull of the vertices $\{v_0,\cdots ,v_{n} \}$ where $v_i=(0,...,1,...,0)\in\mathbb{R}^{n+1}$ with the lone $1$ in the $(i+1)^{th}$ coordinate. 

  The \textit{softmax function} on $n$ vertices $softmax:R^{n} \longrightarrow Int (\Delta_{n-1}) \subset R^{n}$, 
is defined by the composition $S \circ Exp  $ where $Exp : \mathbb{R}^n\to (\mathbb{R}^+)^n  $ is defined by : $ Exp ( x_1,\cdots,x_n  ) = ( \exp( x_1), \cdots,\exp( x_n)  ) $, and $S :\mathbb{R}^n \to \Delta_{n-1} $ is defined by :$$ S ( x_1,\cdots,x_n  ) = ( x_1/\sum_{i=1}^n x_i, \cdots, x_n /\sum_{i=1}^n x_i ) $$. 

 A network $Net$ is said to be a \textit{softmax classification neural network } with $n$ labels if the final layer of $Net$ is softmax function with $n$ vertices. Usually $n$ is the number of labels in the classification problem.  Each vertex $v_i$ in $\Delta_{n-1}$ corresponds to precisely one label $l_{i+1} \in \mathcal{Y} $ for $0 \leq i \leq n-1 $.

For an input $x\in \mathcal{D}$ the point $Net(x)$ is an element of $\Delta_{n-1}$. By definition, the point $x$ is assigned to the label $l_{i+1}$ by the neural network if and only if $Net(x) \in Int(VC( v_i ))$ where $VC(C)$ denotes the Voronoi cell of the set $C$ and $Int(A)$ denotes the interior of a set $A$. This immediately yields the following theorem.

\begin{theorem}
\label{tm}
Let $(D,h,g)$ by a topologically labeled data with, $h: D \to \mathcal{D} \subset R^{d_{in}}$ and $g: \mathcal{D}_L \subset \mathbb{R}^{d_{in}} \to \{l_1, \cdots l_n \}$. A softmax classification neural network $Net : \mathbb{R}^{d_{in}} \to Int(\Delta_{n-1}) $ separates $(D,h,g)$ if and only if $Net(\mathcal{D}_{i+1}) \subset Int( VC(v_i)) $ for $0\leq i \leq n-1$.
\end{theorem}

Finally, to answer the question about the ability of a neural network to separate a topologically labeled data, we combine the result we obtained from Theorem \ref{generalization TLD} with the universality of neural networks \cite{cybenko1989approximations,hanin2017approximating,lu2017expressive} \footnote{The universal approximation theorem is available in many flavors : one may fix the depth of the network and vary the width or the other way around.}. The universality of neural networks essentially states that for any continuous function $f$ we can find a network that approximates it to an arbitrary precision\footnote{The closeness between functions is with respect to an appropriate functional norm. See \cite{cybenko1989approximations,lu2017expressive} for more details. }. Hence we conclude that any topologically labeled data can effectively be separated by a neural network. 

\section{Shape of Data and Neural Networks}

We end our discussion by briefly showing how the shape of input data is essential when deciding on the architecture of the neural network. Theorem \ref{2222}  that if we are not careful about the choice of the first layer function of a network then we can always find a topologically labeled data that cannot be separated by this network.

\begin{theorem}
\label{2222}
Let $Net$ be neural  network  of  the  form :
$Net=Net_1 \circ f_1$
with $f_1:\mathbb{R}^n\longrightarrow \mathbb{R}^k$ such that $f_1(x)= \sigma( W(x)+b)$ and $k<n$ and $Net_1 : \mathbb{R}^k\longrightarrow \mathbb{R}^d $ is an arbitrary net. Then there exists a topologically labeled data $(D,h,g)$ with $h:D \to \mathcal{D} \subset \mathbb{R}^n$  and $g: \mathcal{D}_L \subset \mathcal{D}\to \mathbb{R}^d $ that is not separable by $Net$.
\end{theorem}
\textbf{Proof.}
Let $D = \mathcal{D}= \{x\in\mathbb{R}^n, ||x||\leq 2 \}$.  Let  $\mathcal{D}_L =\mathcal{D}_1 \cupdot \mathcal{D}_2 $ where $\mathcal{D}_1=\{x\in\mathbb{R}^n, ||x||\leq 0.9 \}$ and $\mathcal{D}_2=\{x\in\mathbb{R}^n, 1 \leq ||x||\leq 2 \}$.  Choose $g:\mathcal{D}_L \longrightarrow \{l_1,l_2\}$ such that $g(\mathcal{D}_1)=l_1$ and $g(\mathcal{D}_2)=l_2$. Let $f_1$ be a function as defined in the Theorem. The matrix $W : \mathbb{R}^n \longrightarrow \mathbb{R}^k $ where $k < n$ has a nontrivial kernel. Hence, there is a non-trivial vector $v \in \mathbb{R}^n$ such that $W(v)=0$. Choose a point $p_1 \in \mathcal{D}_1 $ and a point $p_2 \in \mathcal{D}_2$ on the line that passes through the origin and has the direction of $v$. We obtain $W(p_1)=W(p_2)=0$. In other words, $f_1(p_1)=f_1(p_2)$. Hence $Net(p_1)=Net(p_2)$ and hence $Net(\mathcal{D}_1) \cap Net(\mathcal{D}_2) \neq \emptyset $ and so we cannot find two embedded disks that separate the sets $Net(\mathcal{D}_1)$, $Net(\mathcal{D}_2)$.

Note that in Theorem \ref{2222} the statement is independent of the depth of the neural  network. This is also related to the work \cite{johnson2018deep} which shows that skinny neural networks are not universal approximators. This is also related to the work in \cite{nguyen2018neural} where is was shown that a network has to be wide enough in order to successfully classify the input data.


To demonstrate the role that the topology of data may play in regard to the architecture of a neural network we end our discussion by considering the following example. Let $Net$ be a neural network given by the composition $Net=f_6\circ f_5\circ f_4\circ f_3\circ f_2\circ f_1$. For $1\leq i \leq 5$ maps are given by $f_i := Relu(W_i(x)+b_i)$ such that $W_1:\mathbb{R}^2\to \mathbb{R}^5 $, $W_2:\mathbb{R}^5\to \mathbb{R}^5 $, $W_3:\mathbb{R}^5\to \mathbb{R}^2 $
and $W_j : \mathbb{R}^2\to \mathbb{R}^2 $ for $ 4 \leq j \leq 5 $. Finally, the function, $f_6 = softmax(W_6(x)+b_6)$  where  $W_6:\mathbb{R}^2\to \mathbb{R}^2 $. 

We train this network on the annulus dataset given in the top left Figure in \ref{annulus}. In Figure \ref{annulus} we also trace the activations as demonstrated in Figure \ref{annulus}. In the Figure we visualize the activations in higher dimension by projecting them using Isomap \cite{tenenbaum2000global} to $\mathbb{R}^3$. Our choice of this algorithm as a dimensionality reduction algorithm is driven by the fact that the dataset we work with here is essentially a manifold; as such, projecting the space to a lower dimension with the Isomap algorithm should preserve most of the topological and geometric structure of the this space.  

\begin{figure}[h]
  \centering
   {\includegraphics[width=1\textwidth]{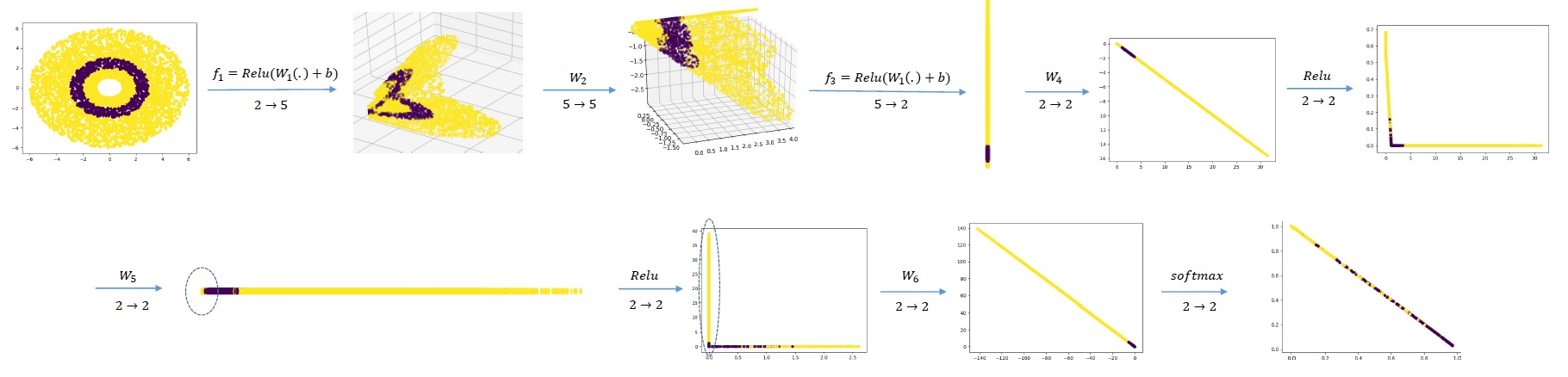}
    \caption{The topological operations performed by a network on data sampled from the annulus and colored by two lables.  }
  \label{annulus}}
\end{figure}
Inspecting the activations in Figure \ref{annulus} we make the following observation: 

\begin{enumerate}
    \item A neural network can collapse the topological space either using the nonlinear $Relu$ or by utilizing the linear part of a given layer function. This is the case with the map $f_3: \mathbb{R}^5\longrightarrow \mathbb{R}^2$. While the linear component is a projection onto $\mathbb{R}^2$, the network "chose" to project the space into $1-$ manifold since the second dimension is not needed for the final classification.
    \item  Note that the yellow components are separated by the purple one, and in order to map both of these parts to the same part of the space, the net has to glue these two parts together. Indeed, the neural network quotients parts of the space as it sees it necessary. This is visible in $W_5$, which acts as a projection, and again $W_6$. 
\end{enumerate}

\bibliographystyle{abbrv}
\bibliography{refs_2}

\end{document}